\documentclass[conference]{IEEEtran}
\IEEEoverridecommandlockouts
\usepackage{cite}
\usepackage{amsmath,amssymb,amsfonts,amsthm,mathtools,wasysym}
\usepackage{algorithmic}
\usepackage{graphicx}
\usepackage{textcomp}
\usepackage{csvsimple}
\usepackage{xcolor}
\usepackage{listings}
\usepackage{pgfplots}
\usepackage{soul}

\usetikzlibrary{calc,positioning,shapes}
\usetikzlibrary{arrows.meta,positioning}
\usetikzlibrary{patterns,decorations.pathmorphing,decorations.markings}
\usetikzlibrary{svg.path}
\pgfplotsset{compat=newest}

\usepackage{ifthen}
\usepackage[hidelinks]{hyperref}

\definecolor{color0}{rgb}{0.12156862745098,0.466666666666667,0.705882352941177}
\definecolor{color1}{rgb}{1,0.498039215686275,0.0549019607843137}
\definecolor{color2}{rgb}{0.172549019607843,0.627450980392157,0.172549019607843}
\definecolor{color3}{rgb}{0.83921568627451,0.152941176470588,0.156862745098039}
\definecolor{color4}{rgb}{0.580392156862745,0.403921568627451,0.741176470588235}
\definecolor{color5}{rgb}{0,0,0}

\def\BibTeX{{\rm B\kern-.05em{\sc i\kern-.025em b}\kern-.08em
    T\kern-.1667em\lower.7ex\hbox{E}\kern-.125emX}}
    
\newcommand{\todo}[1]{\textcolor{red}{#1}}

\newcommand{\N}{\ensuremath\mathbb{N}}

\newcommand{\R}{\ensuremath\mathbb{R}}
\newcommand{\C}{\ensuremath\mathbb{C}}

\newcommand{\mexp}[0]{\mathrm{e}}

\newcommand{\dt}[0]{\frac{\mathrm{d}}{\mathrm{d}t}}
\newcommand{\ds}{\,\mathrm{d}s}

\newcommand{\finalTime}{t_{\mathrm{f}}}
\newcommand{\timeInt}{\mathbb{T}}

\theoremstyle{plain}
\newtheorem{theorem}{Theorem}[section]

\newtheorem{lemma}[theorem]{Lemma}

\newtheorem{remark}[theorem]{Remark}

\newtheorem{assumption}[theorem]{Assumption}

\newtheorem{problem}[theorem]{Problem}

\newboolean{peerReview}
\setboolean{peerReview}{false}

\begin{document}

\title{Certified machine learning: A posteriori error estimation for physics-informed neural networks
\ifthenelse{\boolean{peerReview}}{}{\thanks{The authors acknowledge funding from the DFG under Germany's Excellence Strategy -- EXC 2075 -- 390740016. Both authors are thankful for support by the Stuttgart Center for Simulation Science (SimTech).}
}}

\ifthenelse{\boolean{peerReview}}{%
	\author{\IEEEauthorblockN{Anonymous Authors}}
}{%
\author{\IEEEauthorblockN{Birgit Hillebrecht}
\IEEEauthorblockA{\textit{Stuttgart Center for Simulation Science} \\
\textit{University of Stuttgart}\\
Stuttgart, Germany \\
birgit.hillebrecht@simtech.uni-stuttgart.de\\
ORCID: 0000-0001-5361-0505}
\and
\IEEEauthorblockN{Benjamin Unger}
\IEEEauthorblockA{\textit{Stuttgart Center for Simulation Science} \\
\textit{University of Stuttgart}\\
Stuttgart, Germany \\
benjamin.unger@simtech.uni-stuttgart.de\\
ORCID: 0000-0003-4272-1079}
}
}

\maketitle

\begin{abstract}
	Physics-informed neural networks (PINNs) are one popular approach to incorporate a priori knowledge about physical systems into the learning framework. PINNs are known to be robust  for smaller training sets, derive better generalization problems, and are faster to train. In this paper, we show that using PINNs in comparison with purely data-driven neural networks is not only favorable for training performance but allows us to extract significant information on the quality of the approximated solution. Assuming that the underlying differential equation for the PINN training is an ordinary differential equation, we derive a rigorous upper limit on the PINN prediction error. This bound is applicable even for input data not included in the training phase and without any prior knowledge about the true solution. Therefore, our a posteriori error estimation is an essential step to certify the PINN. We apply our error estimator exemplarily to two academic toy problems, whereof one falls in the category of model-predictive control and thereby shows the practical use of the derived results.
\end{abstract}


\section{Introduction}
Recently, one of the fundamental limitations of machine learning (ML) methods, namely the inability to introduce a priori knowledge about the underlying dynamical system reflected by a partial differential equation (PDE) or ordinary differential equation (ODE), has been overcome by introducing physics-informed machine learning and thus physics-informed neural networks (PINNs) \cite{KarKLPWY21,RaiPK19,LeiBHP21}. The benefit of leveraging this additional information is currently mainly limited to avoiding non-plausible results \cite{KarKLPWY21}, regularizing the method to be robust for small datasets, and improving the performance of the ML method \cite{AlbTC19}. This paper extends these benefits by deriving rigorous a posteriori error bounds if the system is goverened by an ODE. These error estimators subsequently enable the \emph{certification}\footnote{We adopt the terminology \emph{certification} from the model reduction community, see for instance \cite{HesRS16}.} of the ML prediction, i.e., by assessing the quality of the ML prediction, and thereby open up ML to questions that must be answered with quantitative predictions. This is especially useful since no knowledge about the real solution is necessary to evaluate the error estimator. To lower the computational burden, we thereafter relax the requirement of a rigorous upper bound and derive computationally efficient error indicators by learning the error estimator with another network. Using a suitable weighting, we still obtain a smooth upper bound on the error estimator.

\paragraph*{Literature review} 
We present a rigorous and computable estimate of the prediction error of a specific instance of a neural network. This is unlike the standard literature on approximating a function within the class of neural networks, for instance, discussed in \cite{HorSW89,LesYPS93,Bar94,Pin08,LuMMK21}. In particular, our error estimator can be applied regardless of the chosen network architecture. Moreover, we do not aim to construct a neural network capable of reproducing a particular function with a prescribed tolerance, see e.g.~\cite{CaoXX08} and the references therein. An estimation of the prediction error as in our contribution is studied in \cite{CorB19,CorB19b}. Nevertheless, these are statistical approaches (in contrast to our guaranteed error bounds) and rely on sufficiently rich training data to obtain confidence intervals for the prediction error. They thus cannot be used as rigorous error bounds. Closely related to our approach is the work \cite{MinR21}, in which a dual weighted residual estimator is developed and used as a stopping criterion during the network training. Again, we emphasize that this is not a rigorous upper bound on the ML prediction error in contrast to our contribution.

\paragraph*{Organization}
In section~\ref{sec::problemSetting} we describe the problem setting to concisely derive our main result, namely the a~posteriori quantification of the ML prediction error for ODEs in section~\ref{sec::aposteriori_error}. We discuss in section~\ref{sec::appl_Pinn} how a~priori information obtained during the training of the PINN can be leveraged for the computation of the error bound. Our framework is illustrated with two numerical examples from different applications in section~\ref{sec::appl_Pinn}: one simple 1D example and the inverted pendulum as a standard optimization problem. The first example investigates and highlights the dependency of the error estimator on design decisions. The latter one demonstrates the applicability to more complex, non-linear applications, highlighting the potential relevance for model-predictive control applications as discussed in \cite{AntCSSJH21,NicKFU21}. We conclude our presentation with a discussion of our results and future research directions in section~\ref{sec::discussion}.

\paragraph*{Notation}
The real numbers and the set of positive real numbers are denoted with $\R$ and $\R_+ \vcentcolon= \{x\in \R \mid x> 0\}$, respectively.  For a complex number $\lambda\in\C$ we use the notation $\Re(\lambda)$ to denote the real part of $\lambda$.  The set of $n\times m$ matrices with real coefficients is denoted with $\R^{n\times m}$. 
For the derivative w.r.t.~to time we use interchangeably either the classical notation $\dt$ or the dot-notation prevalent in the dynamical systems literature, i.e., $\dot{g}(t) \vcentcolon= \dt g(t)$.  The derivative with respect to the variable $x = (x_1,\ldots,x_n)\in\R^n$ is denoted by $\nabla_x f (x) = \sum_{i=1}^n \tfrac{\partial}{\partial x_i} f(x) \cdot \mathbf{i}$ for $x \in \R^n$ with $\mathbf{i}\in \R^n$ being the unit vector in the $i$-th direction and $\tfrac{\partial}{\partial x_i} f$ the partial derivative of $f$ with respect to the $i$-th component of~$x$.
We use the symbol $\|\cdot\|$ to denote the Euclidean 2-norm. We emphasize that in the theoretical part, any other norm can be used. If the symbol $\|\cdot\|$ is used on a set it describes the number of elements in this set.

\section{Problem description}
\label{sec::problemSetting}

On the time interval $\timeInt = [0,\finalTime)$ with $\finalTime\in\R_+\cup\{\infty\}$ we consider the initial value problem
\begin{equation}
	\label{eq::ode_base}
	\begin{aligned}
		\dot{x}(t) &= f(t, x(t)), & t\in\timeInt \\
		x(0) &= x_0
	\end{aligned}
\end{equation}
with continuous right-hand side $f\colon \timeInt\times\R^{n}\to\R^n$ and initial value $x_0\in\R^n$. To ensure with the theorem of Picard-Lindelöf that the initial value problem has a unique solution for every initial value, we make use of  the following assumption throughout the manuscript.

\begin{assumption}
	\label{ass::Lipschitz}
	The function $f$ in~\eqref{eq::ode_base} is continuous. Besides, it is Lipschitz-continuous with respect to the second argument, i.e., there exists a constant $L>0$ such that
	\begin{displaymath}
		\|f(t,x)-f(t,y)\|\leq L\|x-y\|
	\end{displaymath}
	for all $t\in\timeInt$ and all $x,y\in\R^n$.
\end{assumption} 

With this assumption a unique solution of the initial value problem~\eqref{eq::ode_base} is guaranteed, so that we can define the flow map
\begin{equation}
	\label{eq::flow_map}
	\varphi\colon\timeInt\times \R^n \to \R^n
\end{equation}
that maps the time $t\in\timeInt$ and the initial value $x_0$ to the solution at time $t$, i.e., $x(t) = \varphi(t,x_0)$. Note that in most applications, an explicit expression for $\varphi$ is not available and evaluation of~$\varphi$ is only possible via suitable approximation techniques.

The learning process for the initial value problem is now a two-step process. First, define a suitable ML candidate function
\begin{displaymath}
	 \hat{\varphi}\colon\timeInt\times\R^n\times\R^k\to\R^n,\qquad (t,x_0,\omega)\mapsto \hat{\varphi}(t,x_0,\omega)
\end{displaymath}
with parameter vector $\omega\in\R^k$. If for instance a deep neural network is chosen, then $\hat{\varphi}$ encodes the number of layers and neurons as well as the activation function, and $\omega$ represents the weights for the network. Our standing assumption throughout the manuscript is that the candidate function $\hat{\varphi}$ is sufficiently smooth. Second, find $\omega^\star$ such that $\hat{\varphi}(\cdot,\cdot,\omega^\star)$ is a good approximation of the flow map~\eqref{eq::flow_map} by minimizing a suitable loss function $\mathcal{L}\colon\R^k\to\R$.  Defining $\hat{x} \vcentcolon= \hat{\varphi}(\cdot,x_0,\omega^\star)$, the ML prediction error is then given as
\begin{equation}
	\label{eq::ml_error}
	e(t) \vcentcolon= x(t)-\hat{x}(t)
\end{equation}
for all $t\in\timeInt$. The second step typically relies on existing data, which in our case corresponds to evaluations of $x$ on a discrete subset of $\timeInt$.

\begin{problem}
	\label{problem:certification}
	Rigorously quantify the ML prediction error~\eqref{eq::ml_error} for any $t\in\timeInt$ at low computational cost without computing the true solution.
\end{problem}

\section{A Posteriori Error Estimation}
\label{sec::aposteriori_error}
Assume that we already have trained our network and that for some initial value $x_0\in\R^n$ we have computed the machine learning approximation $\hat{x}$ of the initial value problem~\eqref{eq::ode_base}. Since our candidate function $\hat{\varphi}$ is assumed to be smooth, we can define the residual\footnote{We use the notation $\mathcal{R}_{\hat{\varphi}}(t)$ to indicate that this is the residual for the ML candidate function $\hat{\varphi}$, which additionally depends on the parameter $\omega$.} 
\begin{equation}
	\label{eq::residual}
	\mathcal{R}_{\hat{\varphi}}(t) \vcentcolon= \dot{\hat{x}}(t) - f(t,\hat{x}(t)).
\end{equation}
Herein, the time derivative may be computed efficiently via automatic differentiation, see for instance \cite{BayPRS18}.
With these preparations, we are now ready to formulate our first main result towards the solution of Problem~\ref{problem:certification}.

\begin{theorem}
	\label{thm::core_result}
	Suppose $f$ in~\eqref{eq::ode_base} satisfies Assumption~\ref{ass::Lipschitz} and the ML candidate function $\hat{\varphi}$ is sufficiently smooth. For any continuous function $\delta\colon\timeInt\to\R_+$ with
	\begin{equation}
		\label{eq::upper_bound_residual}
		\|\mathcal{R}_{\hat{\varphi}}(t)\| \leq \delta(t)
	\end{equation}
	define 
	\begin{equation*}
		I(t,\delta)\vcentcolon=\int_{0}^t \mathrm{e}^{L(t-s)}\delta(s)\ds,
	\end{equation*}
	where $L$ is the Lipschitz constant from Assumption~\ref{ass::Lipschitz}.
	Then the ML prediction error~\eqref{eq::ml_error} satisfies
	\begin{equation}
		\label{eq::fund_lemma_stat}
		|| e(t) || \leq \mathrm{e}^{L t} \|x_0-\hat{x}(0)\| + I(t,\delta).
	\end{equation}
\end{theorem}

\begin{proof}
	Let $\hat{x}_0 \vcentcolon= \hat{x}(0)$. Assumption~\ref{ass::Lipschitz} together with the smoothness of $\hat{\varphi}$ implicates that $\mathcal{R}_{\hat{\varphi}}$ is continuous. Then, Assumption~\ref{ass::Lipschitz} implies that $\hat{x}$ is the unique solution of the perturbed initial value problem
	\begin{equation}
		\label{eq::ode_perturbed}
		\begin{aligned}
		\dot{\hat{x}}(t) &= f(t,\hat{x}(t)) + \mathcal{R}_{\hat{\varphi}}(t), & t\in\timeInt, \\
		\hat{x}(0) &= \hat{x}_0.
		\end{aligned}
	\end{equation}
	The result is thus an immediate consequence of \cite[Ch.~I, Variant of Thm.~10.2]{HaiNW08}.
\end{proof}

Using the smoothness of the ML candidate function~$\hat{\varphi}$ and Assumption~\ref{ass::Lipschitz}, we can always use $\delta(t) \vcentcolon= \|\mathcal{R}_{\hat{\varphi}}(t)\|$ in Theorem~\ref{thm::core_result}. However, with this choice we do not expect~$\delta$ to be continuously differentiable, which is necessary in our forthcoming Lemma~\ref{lemma::trpz_rule}. This is the main reason for presenting Theorem~\ref{thm::core_result} in its current form.

\begin{remark}
    Theorem~\ref{thm::core_result} can be easily generalized for time dependent Lipschitz parameters 
    \begin{equation*}
        || f(t, x) - f(t, y) || \le \ell(t) || x-y||
    \end{equation*}
    by replacing the exponent $Lt$ by $L(t)=\int_0^t \ell(s)\ds$ in~\eqref{eq::fund_lemma_stat}.
\end{remark}

If the differential equation in~\eqref{eq::ode_base} is linear and time-invariant, i.e., the right-hand side of~\eqref{eq::ode_base} is given by
\begin{equation}
	\label{eq::ode_linear}
	f(t,x) = Ax
\end{equation}
for some matrix $A\in\R^{n\times n}$, then we can further improve the error estimation as follows.

\begin{theorem}
	\label{thm::linear_rhs}
	Suppose that the right-hand side of~\eqref{eq::ode_base} is linear, i.e., given by~\eqref{eq::ode_linear}, and $\hat{\varphi}$ is sufficiently smooth. 
	Then there exist constants $\alpha \in\R$, $\beta\in\R_+$ such that for any continuous function $\delta\colon\timeInt\to\R$ with
	\begin{displaymath}
		\|\mathcal{R}_{\hat{\varphi}}(t)\| \leq \delta(t),
	\end{displaymath}
	the ML prediction error satisfies
	\begin{equation}
		\label{eq::improvedErrLinear}		
		\resizebox{.98\linewidth}{!}{$\begin{aligned}
		|| e(t)  || \leq \beta \left(\mathrm{e}^{\alpha t}\|x_0 -\hat{x}(0)\| + \int_{0}^t \mathrm{e}^{\alpha(t-s)}\delta(s)\ds\right).
		\end{aligned}$}
	\end{equation}
\end{theorem}

\begin{proof} \todo{Corrected}
	Subtracting \eqref{eq::ode_perturbed} from \eqref{eq::ode_base} yields the error dynamics
	\begin{equation}
		\label{eq::ode_errorDynamics}
		\begin{aligned}
		\dot{e}(t) &= Ae(t) + \mathcal{R}_{\hat{\varphi}}(t), & t\in\timeInt, \\
		e(0) &= x_0 - \hat{x}_0.
		\end{aligned}
	\end{equation}
	Its unique solution is given by 
	\begin{displaymath}
		e(t) = \mathrm{e}^{At}(x_0-\hat{x}_0) + \int_0^t \exp(A(t-s))\mathcal{R}_{\hat{\varphi}}(s)\ds.
	\end{displaymath}
	Since for any matrix $A$ there exist constants $\alpha \in \R$, $\beta\in\R_+$ such that $\|\exp(At)\|\leq \beta \mathrm{e}^{\alpha t}$, we obtain
	\begin{equation*}
		\begin{aligned}
			\|e(t) \| &\leq  \|  \mathrm{e}^{At} (x_0-\hat{x}_0)  \| + \int_0^t \| \exp(A(t-s))\mathcal{R}_{\hat{\varphi}}(s) \|\ds \\
			&\le \beta\mathrm{e}^{\alpha t} \| x_0 -\hat{x}_0\| +  \beta\int_0^t \mathrm{e}^{\alpha(t-s)} \delta(s) \ds
		\end{aligned}
	\end{equation*}
	which concludes the proof.
\end{proof}

Let us emphasize that in many applications $\alpha$ can be chosen significantly smaller than the Lipschitz constant $L$ from Assumption~\ref{ass::Lipschitz}. Indeed, if the matrix $A$ is diagonalizable, then we can pick $\beta=1$ and $\alpha$ as the spectral abscissa, i.e., we can set
\begin{equation}
	\label{eq::spectral_abscissa}
	\alpha = \max\{\Re(\lambda) \mid \lambda \text{ eigenvalue of $A$}\}.
\end{equation}

\begin{remark}
	\label{rem::extension_PDE}
	Theorem~\ref{thm::linear_rhs} details a first step towards a generalization of the error estimator to PDEs. Assume that~\eqref{eq::ode_base} constitutes a linear abstract ODE defined on an appropriate Banach space and that $A$ in~\eqref{eq::ode_linear} is the generator of a strongly continuous semigroup (cf.~\cite{Paz83}). Then, assuming sufficiently smooth initial data,  the error estimation~\eqref{eq::improvedErrLinear} is still valid. Nevertheless, an additional approximation step is required to evaluate the norm on the underlying infinite-dimensional space. This is subject to further research.
\end{remark}

To obtain a computable error bound, we have to estimate the integral $I(t,\delta)$ in~\eqref{eq::fund_lemma_stat}, which is done here exemplarily by trapezoidal rule. More specifically, we define $\hat{I}_n(t,\delta)$ as an approximation of the integral $I(t,\delta)$ with composite trapezoidal rule with $n \in \N$ subintervals, i.e.,
\begin{equation*}
	\hat{I}_n(t,\delta) \vcentcolon= \frac{t}{2n} \mathrm{e}^{Lt}\sum_{i=0}^{n-1}  \left( \mathrm{e}^{-L \tfrac{i+1}{n}t} \delta(\tfrac{i+1}{n}t) + \mathrm{e}^{-L\cdot \tfrac{i}{n}t} \delta(\tfrac{i}{n}t) \right).
\end{equation*} 

\begin{lemma}
	\label{lemma::trpz_rule}
    Assume that the right-hand side $f$ in~\eqref{eq::ode_base} satisfies Assumption~\ref{ass::Lipschitz} and that the ML candidate function $\hat{\varphi}$ is sufficiently smooth. Let $\delta \in \mathcal{C}^3(\timeInt, \R)$ with
    \begin{equation*}
        \|\mathcal{R}_{\hat{\varphi}}(t)  \| \le \delta(t)
    \end{equation*} 
    and $\| \frac{\mathrm{d}^2}{\mathrm{d}s^2} \left( \mexp^{-Ls} \delta(s) \right)\| \le K$ for $s \in \mathbb{T}$. Then 
    \begin{equation*}
        \| \hat{x}(t) - x(t) \| \le \mexp^{Lt} \|x_0 - \hat{x}(0)\| + \hat{I}_n(t, \delta) + E_\mathrm{Int} \, \text{,} 
    \end{equation*}
    with $\hat{I}_n(t,\delta)$ is the composite trapezoidal rule approximation of $I(t,\delta)$ using $n$ subintervals. The error from the numerical integration is encapsulated in 
	\begin{displaymath}
		E_\mathrm{Int} = \mexp^{Lt}\frac{Kt^3}{12 n^2}
	\end{displaymath}
\end{lemma}

\begin{proof}
	Follows directly from the application of composite trapezoidal rule for the integral $J(t, \delta) = I(t, \delta) \cdot e^{-Lt}$ and applying known error bounds \cite{Gau97} to $\hat{J}_n(t,\delta)$ and rescaling the result by multiplication with $e^{Lt}$. 
\end{proof}

Lemma~\ref{lemma::trpz_rule} requires a sufficiently smooth upper limit for the norm of the residual. 
To achieve this, we follow \cite{RamSKA14} and construct a smooth upper limit for $\|\mathcal{R}_{\hat{\varphi}}(t)\|$ as
\begin{equation}
	\label{eq::delta_construction}
	\delta(t) \vcentcolon= \sqrt{\|\mathcal{R}_{\hat{\varphi}}(t)\|^2 +  \mu^2}
\end{equation}
with $\mu\in\R$. Due to the numerical experiments (details can be found in section~\ref{subsec::Appl_1DODE}), choosing $\mu \vcentcolon= \tfrac{1}{10}\overline{\| \mathcal{R_{\hat{\varphi}}} \|}$ has shown to be a suitable candidate. Therein, we use the average deviation of fulfilling the ODE over the set of collocation points $Y_\mathrm{coll}$ defined as
\begin{equation*}
	\overline{\| \mathcal{R}_{\hat{\varphi}} \|} \vcentcolon= \frac{1}{\|Y_\mathrm{coll}\|} \sum_{y \in Y_\mathrm{coll}}\| \mathcal{R}_{\hat{\varphi}}(y.t) \|,
\end{equation*}
wherein $y.t$ denotes the time, which is part of collocation point $y$.
\section{Application for Physics Informed Neural Networks}
\label{sec::appl_Pinn}

The main idea of physics-informed ML, see for instance~\cite{RaiPK19}, is to encode physical information into the ML framework. In PINNs this is achieved by adding an additional term to the loss function which measures how well the ML candidate function satisfies the governing ODE~\eqref{eq::ode_base}. Here, we showcase that this additional term in the loss function can be leveraged for the error estimation.

Suppose we have training data $(t,x_0,\varphi(t,x_0))\in \R^{2n+1}$ aggregated in a finite training data set $Z_{\mathrm{data}}$. For notational convenience we introduce the notation 
\begin{gather*}
	\begin{aligned}
	z.t &\vcentcolon= z_1 \\
	z.x_0 &\vcentcolon= (z_2,..., z_{n+1})\\
	z.x &\vcentcolon= (z_{n+2},..., z_{2n+1}).
	\end{aligned}
\end{gather*}
for any $z = (z_1,\ldots,z_{2n+1})\in Z_\mathrm{data}$.
The contribution to the loss function, which describes how well a training data set is reproduced by the NN, is defined as 
\begin{equation}
	\label{eq::Ldata}
	\mathcal{L}_\mathrm{data}(\omega) = \frac{1}{\|Z_\mathrm{data}\|} \sum_{z \in Z_\mathrm{data}} \| \hat{\varphi}(z.t, z.x_0, \omega) - z.x \|^2.
\end{equation}
For the second contribution, we define a set of collocation points $(t, x_0) \in Y_\mathrm{coll}$ for which the expected output value $x(t) = \varphi(t,x_0)$ is unknown. Their contribution to the loss is defined via
\begin{equation}
	\label{eq::Lphys}
	\begin{aligned}
		\mathcal{L}_\text{phys}(\omega) &= \frac{1}{\|Y_\mathrm{coll}\|} \sum_{y \in Y_\mathrm{coll}} \eta(y.t)\| \mathcal{R}_{\hat{\varphi}}(y.t) \|^2,
	\end{aligned}
\end{equation}
where $\mathcal{R}_{\hat{\varphi}}$ denotes the residual as defined in~\eqref{eq::residual}. Two remarks are in order. First, we emphasize that the residual implicitly depends on the ML parameter $\omega$ and on the initial value~$z.x_0$. Second, given the error estimators derived in Theorems~\ref{thm::core_result} and~\ref{thm::linear_rhs},  we have added an additional weighting term $\eta(t)$ in~\eqref{eq::Lphys}. The choice $\eta\equiv 1$ recovers the classical loss function as introduced in~\cite{RaiPK19}. 
The total loss function used for the training is then given by a suitable linear combination of the loss functions~\eqref{eq::Ldata} and~\eqref{eq::Lphys} with non-negative weights $\gamma_\mathrm{data}, \gamma_\mathrm{phys}$, i.e.,
\begin{equation*}
    \mathcal{L}(\omega) = \gamma_\mathrm{data} \cdot \mathcal{L}_\mathrm{data}(\omega) +\gamma_\mathrm{phys} \cdot \mathcal{L}_\mathrm{phys}(\omega).
\end{equation*}

\subsection{A priori estimates of characteristic quantities}

We immediately notice that although our error estimator is of a posteriori nature, we use the residual information a priori during the training of our PINN. Moreover, we can use the collocations points for an a priori estimate of the constants $L$ and $K$ required in Lemma~\ref{lemma::trpz_rule}. 

The Lipschitz constant $L$ can be approximated by means of automatic differentiation. In more detail, we  evaluate $\nabla_x f(t, x)$
for all collocation points and use singular value decomposition to extract the growth of f~w.r.t.~x at these points. The Lipschitz constant $L$ is then immediately estimated as the maximum of the determined largest singular values. However, this procedure partly reflects the selection of the collocation points. A too sparse usage of collocation points for the determination of $L$ might result in underestimating the parameter and thus invalidating the error estimation. 

To estimate $K$ in Lemma~\ref{lemma::trpz_rule}, i.e., when the trapezoidal rule is used to approximate the integral, we derive auxiliary parameters first. The expected machine learning error at time $t$ can be estimated a priori as
\begin{equation*} 
	E_{\mathrm{ML}}^{\mathrm{exp}}(t, x_0) = \mexp^{L t} \|x_0 - \hat{x}(0)\| + (\mexp^{L t} - 1 ) \frac{\overline{\| \mathcal{R}_{\hat{\varphi}} \|}}{L} .
\end{equation*}
The contribution of the numerical integration error to the total error can now be bounded by requiring that
\begin{equation} 
	\label{eq::contr_E_I}
	E_\mathrm{Int}(t, x_0) \le \varepsilon E_{\mathrm{ML}}^{\mathrm{exp}}(t, x_0).
\end{equation}
for fixed $\varepsilon>0$ . We then estimate the number of required subintervals for the trapezoidal rule as 
\begin{equation} 
	\label{eq::N_SI}
	N_\mathrm{SI}(t, x_0) = \left\lceil \sqrt{\frac{ \mathrm{e}^{Lt} K t^3}{12 E_{\mathrm{ML}}^{\mathrm{exp}}(t, x_0) \varepsilon}}\;\right\rceil,
\end{equation} 
which might be corrected by a posteriori determination of $K$ and~$E_\mathrm{ML}^{\mathrm{exp}}$.	

\subsection{Learning the error estimator}
\label{subsec::nnforerrors}

One of the main benefits of using PINNs for dynamical systems is to avoid tedious and computationally expensive integration over time. However, the introduced error estimator requires detailed knowledge of the temporal behaviour of $\delta$ and thereafter integration of the found function. 

To minimize the computational burden, we propose to learn an error indicator that mimics the error estimator but does not require evaluations of $\delta$ at multiple time points. This means, that we now learn two neural networks, a PINN, which reflects the temporal behaviour of the target system, and another purely data-based NN, which returns the error indicator. In the following, the latter one is referred to as error NN. 

To train the error NN as a purely data-based NN we require an adequate data set consisting of input data and reference data. For this purpose, we randomly generated $N_\mathrm{gen}$ input data points $(t, x_0) \in Y_{\mathrm{gen}}$ within the allowed range and evaluated the PINN. Based on PINNs output data, the value of the a posteriori error estimator $E_\mathrm{pred}$ is determined using conventional integration via trapezoidal rule. This is depicted as the grey path in Fig.~\ref{fig::error_network_training}. The generated data points in $Y_{\mathrm{gen}}$ and the computed $E_\mathrm{pred}$ are then used as input and reference data to train the error neural network, as shown in blue in Fig.~\ref{fig::error_network_training}.

\begin{figure}
	\centering
	\begin{tikzpicture}
		\node (Xdata) {$Y_{\mathrm{gen}}$};
		\node[rectangle, minimum width=2.2cm, fill=black!30, right=0.5cm of Xdata] (PINN) {PINN};
		\draw[-latex,shorten >= 2pt, shorten <= 2pt] (Xdata) -- (PINN);
		\node[right=0.5cmof PINN] (Yres) {$\hat{\varphi}(Y_{\mathrm{coll}})$};
		\draw[shorten >= 2pt, shorten <= 2pt] (Yres) -- (PINN);
		\node[rectangle,minimum width=2.2cm, fill=black!30, right=0.5cm of Yres] (EImpl) {Lemma~\ref{lemma::trpz_rule}};
		\draw[-latex,shorten >= 2pt, shorten <= 2pt] (Yres) -- (EImpl);
		\node[below=0.5cm of EImpl] (Epred) {$E_{\mathrm{pred}}$};
		\node[rectangle, minimum width=2.2cm,fill=color0, below=0.5cm of Epred] (ENN) {Error NN};
		\draw[color0,shorten >= 2pt, shorten <= 2pt] (Epred) -- (EImpl);
		\draw[color0, -latex,shorten >= 2pt, shorten <= 2pt] (Epred) -- (ENN);
		\draw[color0, -latex,shorten >= 4pt, shorten <= 2pt] (Xdata) -- (ENN);
		\node (train) at (6,-1.5) {\textcolor{color0}{train}};
	\end{tikzpicture}
	\caption{Training sequence for the error neural network based on conventional evaluations of the contributions of Lemma~\ref{lemma::trpz_rule} (grey node in the top right corner). The two neural networks, the one reflecting the evolution of the system over time and the other one describing the error are called PINN and Error NN, respectively.}
	\label{fig::error_network_training}
\end{figure}
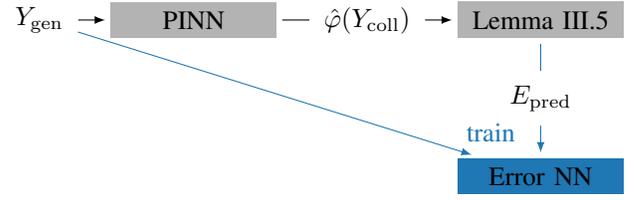

\begin{remark}
	If the underlying dynamic is linear and time-invariant, i.e., as in the setting of Theorem~\ref{thm::linear_rhs}, then instead of learning the error indicator, we could also use a PINN approach to directly learn the error dynamics~\eqref{eq::ode_errorDynamics}. This should be closer to the actual prediction error since we do not introduce any additional approximation steps. Nevertheless, a thorough investigation is subject to further research. 
\end{remark}

\section{Numerical examples}
\label{subsec::numex}

We study our ML certification framework with two simple numerical examples. The first one is a scalar, linear ODE and the second one, is the inverted pendulum on a cart. For the implementation we use the Python framework TensorFlow\footnote{https://www.tensorflow.org/}. For a thorough investigation of the predicted a posteriori error estimator and its contributions we decompose the predicted error as follows
\begin{displaymath}
	\|  x(t) - \hat{x}(t) \| \le E_\mathrm{init}(t) + E_\mathrm{PI}(t) \, ,
\end{displaymath}
wherein the contribution of the error in the initial value to the total error at time $t$ is called 
\begin{displaymath}
	E_\mathrm{init}(t) \vcentcolon= \mathrm{e}^{Lt} \|\hat{x}_0 - x_0 \|.
\end{displaymath}	
The remainder, i.e., the deviation introduced by approximating the ODE is 
\begin{displaymath}
	E_\mathrm{PI}(t) =  \int_0^t \mathrm{e}^{L(t-s)} \delta (s) \ds.
\end{displaymath}	

If we learn the error by a neural network (cf.~subsec.~\ref{subsec::nnforerrors}), the neural network predicted error is denoted with~$E_\mathrm{NN}$.

\vspace{0.2cm}
\noindent\fbox{%
    \parbox{0.48\textwidth}{%
        The code and data used to generate the subsequent results are accessible via
		\begin{center}
			\ifthenelse{\boolean{peerReview}}{%
				\emph{to be inserted after double-blind review}
			}{doi: 10.5281/zenodo.6557796}
		\end{center}
		under MIT Common License.
    }%
}
\vspace{0.2cm}

\subsection{Application for 1D-ODE with fixed initial conditions}
\label{subsec::Appl_1DODE}

The considered problem is 
\begin{equation}
	\label{eq::exampleprob}
	\begin{aligned}
    \dot{x}(t) &= -2x(t) \\
    x(0) &= 2
	\end{aligned}
\end{equation}
with the known solution $x(t) = 2 \mathrm{e}^{-2t}$. This can be learned by physics-informed ML with a relatively low number of epochs. Note that for simplicity, we do not vary the initial value and hence only have to learn the mapping $t \mapsto 2\mathrm{e}^{-2t}$.

For our numerical investigation we use a neural network with one node in input and output layer, respectively, and two hidden layers with 4 nodes each. The neural network is implemented with the hyperbolic tangent as activation function. The network is trained with 200 collocation points and 5000 epochs using the TensorFlow optimizer \texttt{adam}.

For the smooth upper limit of the absolute value of the residual, we use the construction as in~\eqref{eq::delta_construction}, i.e.,
\begin{equation}
	\label{eq::smoothening}
    \delta(t) = \sqrt{ \left(\dot{x} + 2 x \right)^2 + \mu}
\end{equation}
with $\mu = \tfrac{1}{10}\overline{\| \mathcal{R} \|}$. As mentioned previously, for an adequate determination of $\mu$, one has to consider the dependency of $K$ on $\mu$. In Fig.~\ref{fig::Kovermu}, this dependency is depicted. For a small choice of $\mu$, the contribution of the first and second derivative of $\delta(t)$ dominates $K$, whereas for larger $\mu$, the function $\delta(t)$ becomes smoother but larger in absolute value. 

\begin{figure}
	\centering
	\includegraphics[width=\linewidth]{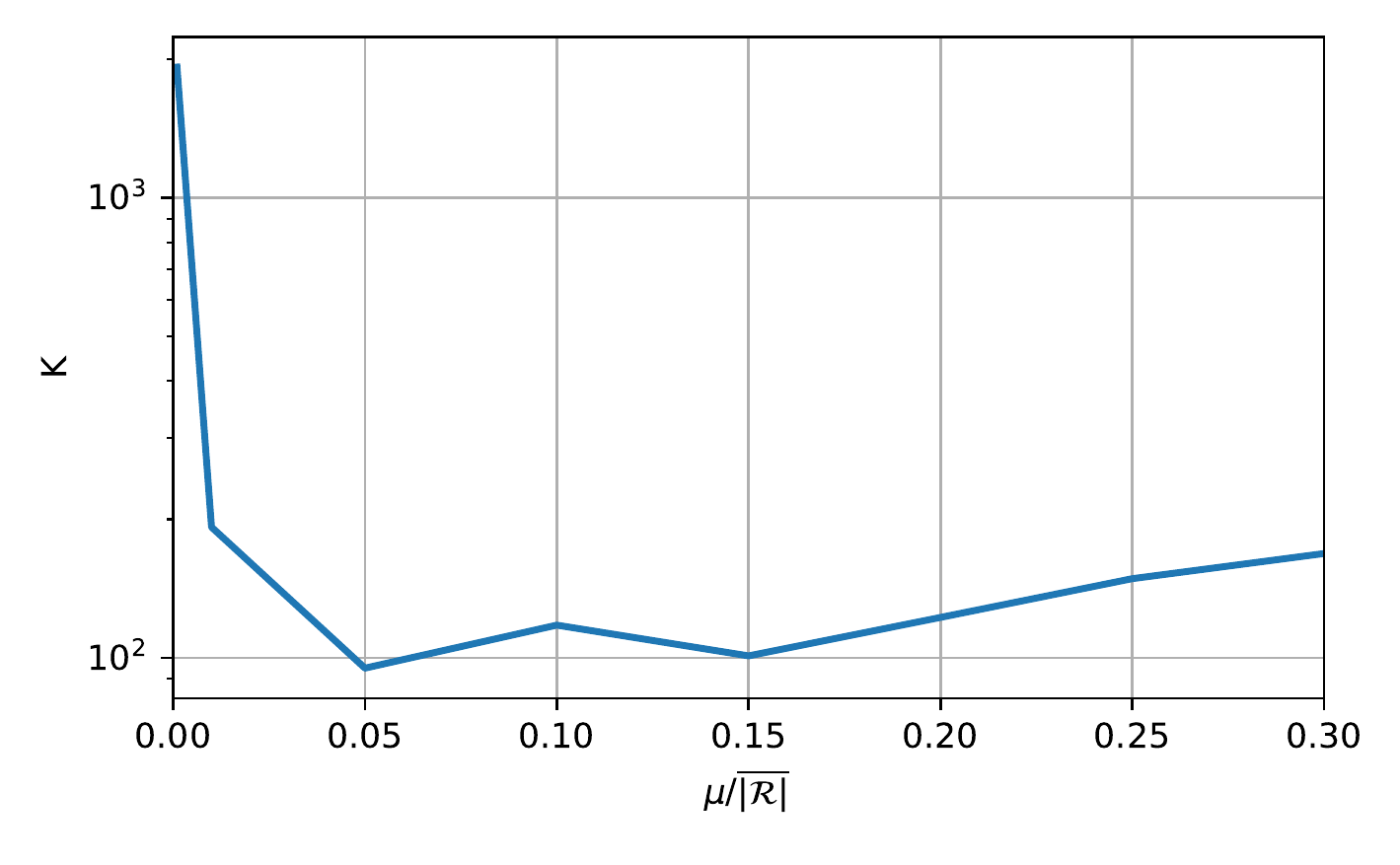}
	\caption{Dependency of computed $K$ on selection of $\mu$ for a neural network with tanh activation function. }
	\label{fig::Kovermu}
\end{figure}

The contribution of the error of the numeric computation of the integral is limited a priori by $\epsilon = 0.33$ (see~\eqref{eq::contr_E_I}). 

\begin{figure}
    \centering
    \includegraphics[width=\linewidth]{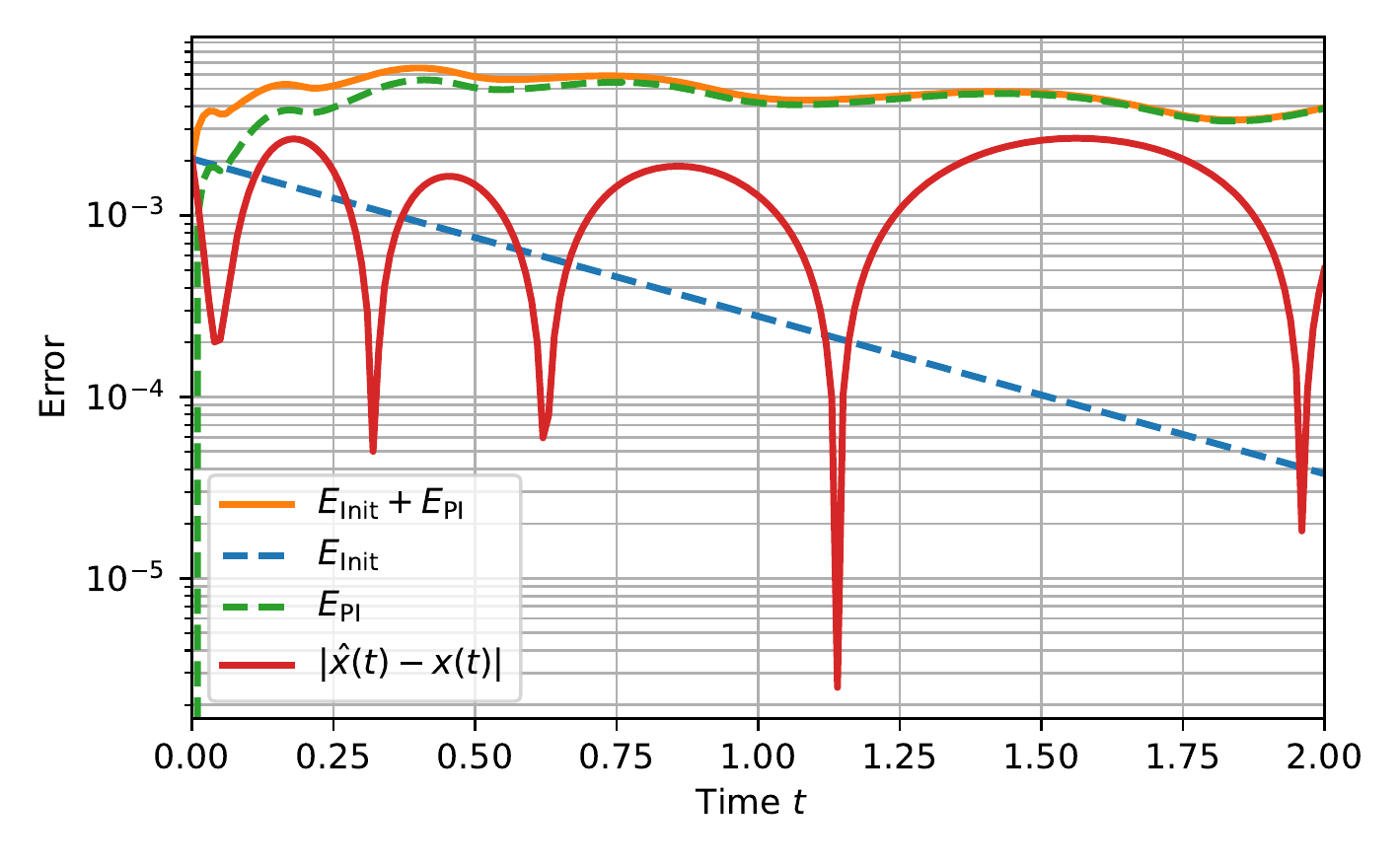}
    \caption{Actual absolute error of the learned solution (red line) and the predicted error by a posteriori error estimation (orange line). The first data point of $E_\mathrm{PI}$ (green dashed line) is out of scope of the figure because it is evaluated to zero at time zero. This is reasonable, since the fulfillment of the ODE plays no role if no time has passed.  \label{fig::abs_vs_pred}}
\end{figure}

Since we are in a linear regime, we can use Theorem~\ref{thm::linear_rhs} instead of Theorem~\ref{thm::core_result}. In more detail, we can use $\alpha=-2$ and $\beta=1$ instead of $L=2$. This significantly improves the error estimation. The predicted and the absolute error are depicted in Fig.~\ref{fig::abs_vs_pred}. Clearly, our estimator is a strict upper bound on the actual error with only a small overestimation. We observe that the contribution of the error in the initial value, i.e., $E_{\mathrm{init}}$ is only significant for very small $t$. Towards the end of the considered time interval, the error introduced by the initial value error is almost negligible. 

\begin{remark}
	When choosing $\eta(t)$ to enhance the importance of small or large times during training, better results might be achieved w.r.t. two quality measures: firstly, the overall approximation error of the PINN might be smaller than without weighting, and secondly, the a posteriori error estimator might be closer to the actual error. These preliminary observations hint at the need for a thorough investigation of a weighting function $\eta(t)$ in the context of certified machine learning and the introduced a posteriori error estimator.
\end{remark}

The number of necessary subintervals $N_\mathrm{SI}$ as defined in~\eqref{eq::N_SI} amounts to approximately 250 for the maximum time in the allowed timerange $t=2$. This highlights the added value by learning an error indicator as proposed in subsection~\ref{subsec::nnforerrors}, which is trained and applied for the 1D example in the following. For the training, we use random times in the allowed domain~$\timeInt$ to generate training data for the neural network representing the error. In the following, training sets consisting of 100, 1000, or 10000 generated data points are used to train a neural network with 2-8 layers with each 4-16 neurons. We use the hyperbolic tangent as activation function and L-BFGS \cite{LiuN89} with 20000 epochs as optimizer. 
Learning the error in this particular case shows to be challenging. Although the magnitude is found easily during learning, even for small NNs (2 layers, 4 neurons) and small learning sets (100 random time points), the detailed evolution of the error can not be represented that easily. All tested combinations of the size of the data set, the number of layers, and the number of neurons showed similar results for the test data as shown in Fig.~\ref{fig::ENN_weightedOrNot}.

Especially in this magnitude, it is more important to preserve the property that the error estimator is an upper limit for the actual error while not increasing the magnitude significantly than to reflect the detailed behavior of the error. This is why we introduced additional weighting during the training of the error NN and penalized an underestimation of the error 
\begin{displaymath}
	E_\mathrm{NN} < E_\mathrm{Init} + E_\mathrm{PI}
\end{displaymath}
stronger (by a factor 1000) than an overestimation. This leads to the neural network modelling a smooth wrapper around the actual error, as shown in Fig.~\ref{fig::ENN_weightedOrNot}.

\begin{figure}
	\centering
	\includegraphics[width=\linewidth]{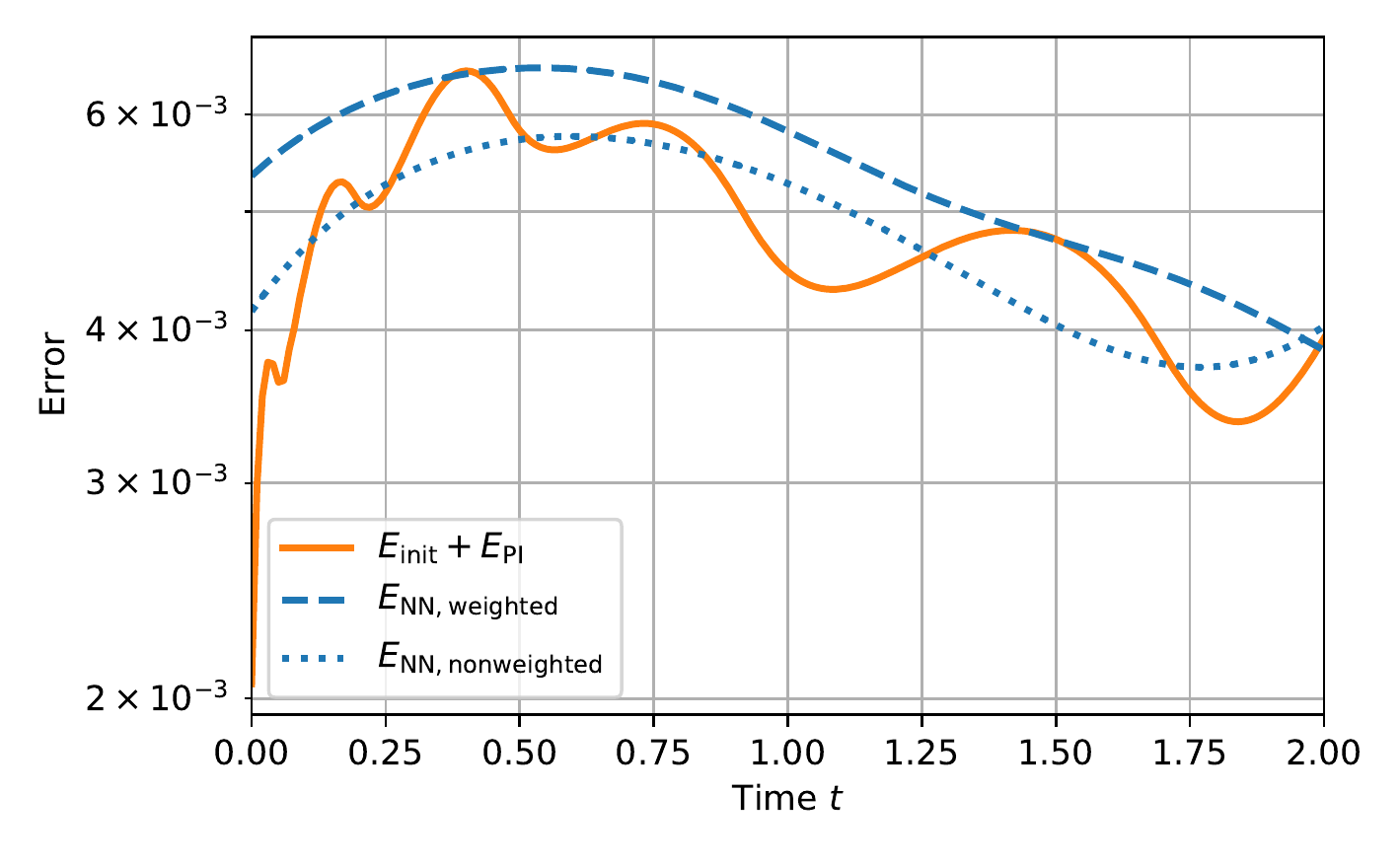}
	\caption{Predicted error with the a posteriori error estimator evaluated by numerical integration (orange solid line) in contrast to error prediction of the neural network (blue lines) trained with artificially generated data (data set size 100). Here, a neural network with 2 layers and 4 neurons in each layer is used. The difference between the two depicted curves is, whether underestimating the error was penalized more than overestimating it. In case of $E_\mathrm{NN, weighted}$ (blue dashed line) it was penalized by multiplying the loss with 1000 for underestimation, while this was 1 for $E_\mathrm{NN, notweighted}$ (blue dotted line).} \label{fig::ENN_weightedOrNot}
\end{figure}

\begin{remark}
	We emphasize that we could also investigate the toy example with an explicit dependency on the initial value. This would introduce an additional input neuron and thus complicate the approximation task slightly. Nevertheless, the computation of the error bound and hence the certification of the approximation can be carried out as presented above. A more realistic example with a varying initial value is discussed next. 
\end{remark}

\subsection{Application for 4D-ODE with dependence on initial conditions}
\label{subsec::InvPendulum}
As a slightly more complex example, we selected the benchmark problem of the inverted pendulum. The control problem to stabilize the upright position has been solved early with neural networks \cite{And89,HunSZG92} and is here used to show the applicability of the error estimator to one of the standard problems governed by ODEs. 

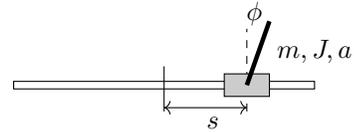
\begin{figure}
	\centering
	\begin{tikzpicture}
		\draw[draw=black] (-2,0) rectangle (2,0.1);
		\draw[] (0,-0.3) -- (0,0.3);
		\draw[dashed] (1.1,-0.3) -- (1.1,0.8);
		\draw[draw=black, fill=black!20] (0.8, -0.1) rectangle ++(0.6,0.3);
		\draw[<->] (0, -0.25) -- (1.1, -0.25);
		\node (s) at (0.65, -0.45) {$s$} ;
		\draw[line width=0.6mm, draw=black] (1.1, 0.05) -- (1.4, 0.9);
		\node (phi) at (1.2, 1) {$\phi$} ;
		\node (p) at (2, 0.5) {$m, J, a$} ;
	\end{tikzpicture}
	\caption{Inverted pendulum as described by \eqref{eq::ODE_inv_pendulum} with the variables of interest being the angle of the pendulum $\phi$ and displacement of the cart~$s$. The parameters $m= 0.3553\, \mathrm{kg}$, $a = 0.42\,\mathrm{m}$, and $J= 0.0361\, \mathrm{kg}/\mathrm{m}^2$ denote the (point) mass of the cart, the distance between the center of mass of the cart and the pendulum $a$, and the mass moment of inertia, respectively.}
	\label{fig::pendulum}
\end{figure}

The mathematical pendulum on a cart is formulated as first-order ODE as 
\begin{equation}
	\label{eq::ODE_inv_pendulum}
	\dt \begin{bmatrix}
		\phi \\
		\dot{\phi} \\
		s \\
		\dot{s}
	\end{bmatrix} =
	\begin{bmatrix}
		\dot{\phi} \\
		\frac{m \mathsf{g} a \sin(\phi) - d \dot{\phi} + ma \cos(\phi) u}{J+ m a^2} \\
		\dot{s} \\
		u
	\end{bmatrix},
\end{equation}
wherein the variables and parameters are used as depicted in Fig.~\ref{fig::pendulum}. Futhermore, $\mathsf{g} = 9.81 \, \mathrm{m}/\mathrm{s}^2$ denotes the gravitational constant and $d = 0.005\, \mathrm{Nms}$ the friction coefficient. The parameter $u$ describes the acceleration of the cart, which we use as the control element. 

The neural network has thereby a 4-dimensional output and a 6-dimensional input consisting of the time $t$, the acceleration of the cart $u$ and the 4-dimensional initial value. As activation function we use $\tanh$ and for the network topology we choose 4 layers with 32 neurons each. For this more complex example we use 100.000 epochs of the L-BFGS optimizer \cite{LiuN89}. We use 10.000 randomly generated collocation points to assert conformance to the ODE and 50 data points. The 10.000 randomly generated collocation points are generated in the allowed domain for the 6-dimensional input data, ranging for 
\begin{align*}
	t &\in \mathbb{T} = [0,0.1]\, \mathrm{s},\\
	x_0 &\in [-\pi, \pi] \times [-6, 6]\,{1}/{\mathrm{s}} \times [-1,1]\, \mathrm{m} \times [-3,3] \,\mathrm{m}/\mathrm{s},\\
	u &\in [-15,15]\,\mathrm{m}/\mathrm{s}^2.
\end{align*}

\begin{figure}
	\centering
	\includegraphics[width=0.97\linewidth]{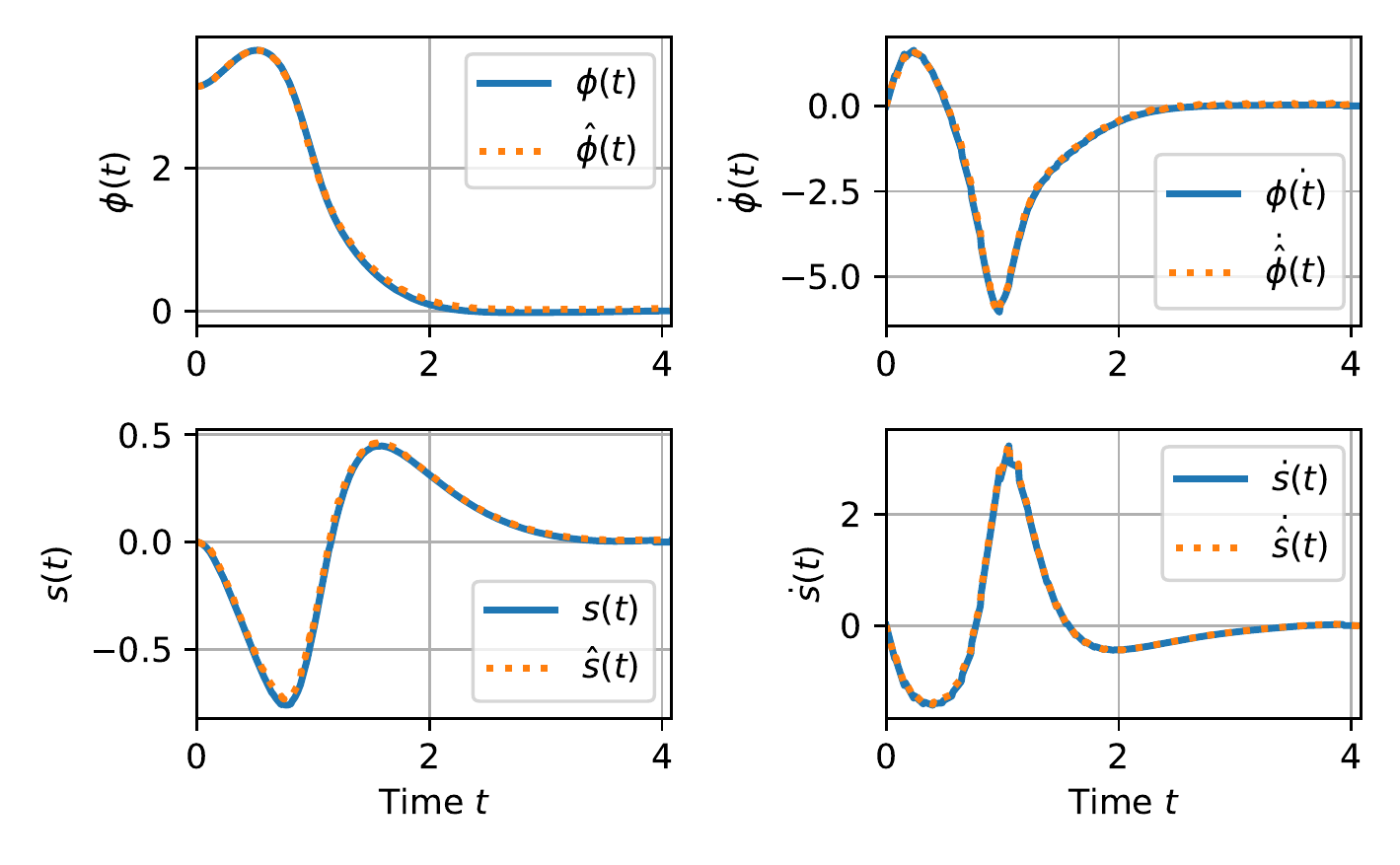}
	\caption{Comparison of the solution generated by the PINN (orange dashed line) with the reference solution of the inverse pendulum dynamics~\eqref{eq::ODE_inv_pendulum} generated by a fine grained forward Euler (blue solid line).}
	\label{fig::phi_s_inv_pendulum}
\end{figure}

Similarly as in model-predictive control, we split the time interval $[0,4]\,\mathrm{s}$ into $50$ control intervals and use a constant external accelleration $u$ of the cart within each subinterval. The control $u$ and the inital values for the control intervals $x_0$ are determined a priori using full discretization. In consequence, we restrict the investigation of our a posteriori error estimator to the investigation of the temporal evolution of the system and the error within each time interval. 
The learned solution agrees well with the reference solution computed with the explicit (forward) Euler-Method with 100 grid points per control interval (see Fig.~\ref{fig::phi_s_inv_pendulum}). 
To evaluate the performance of the error estimation, the first control intervals are shown in Fig.~\ref{fig::first200_abserr}. For the complete investigated interval, the actual error $\|x(t) - \hat{x}(t)\|$ between the reference and the PINN approximated solution of \eqref{eq::ODE_inv_pendulum}, is less than the a posteriori error estimator computed via \ref{eq::fund_lemma_stat}. It is worth noticing that the error estimator is dominated by the error of the initial conditions $E_\mathrm{Init} = \mathrm{e}^{Lt} \|x_0 -\hat{x}(0) \| $. Even though the contributions of $E_\mathrm{PI} = \mathrm{e}^{Lt}(I(t) + E_{I}(t))$ become significant for four control intervals, the ODE is sufficiently fulfilled to keep this error notably smaller than the contribution of the error in the initial conditions. 

\begin{figure}
	\centering
	\includegraphics[width=0.97\linewidth]{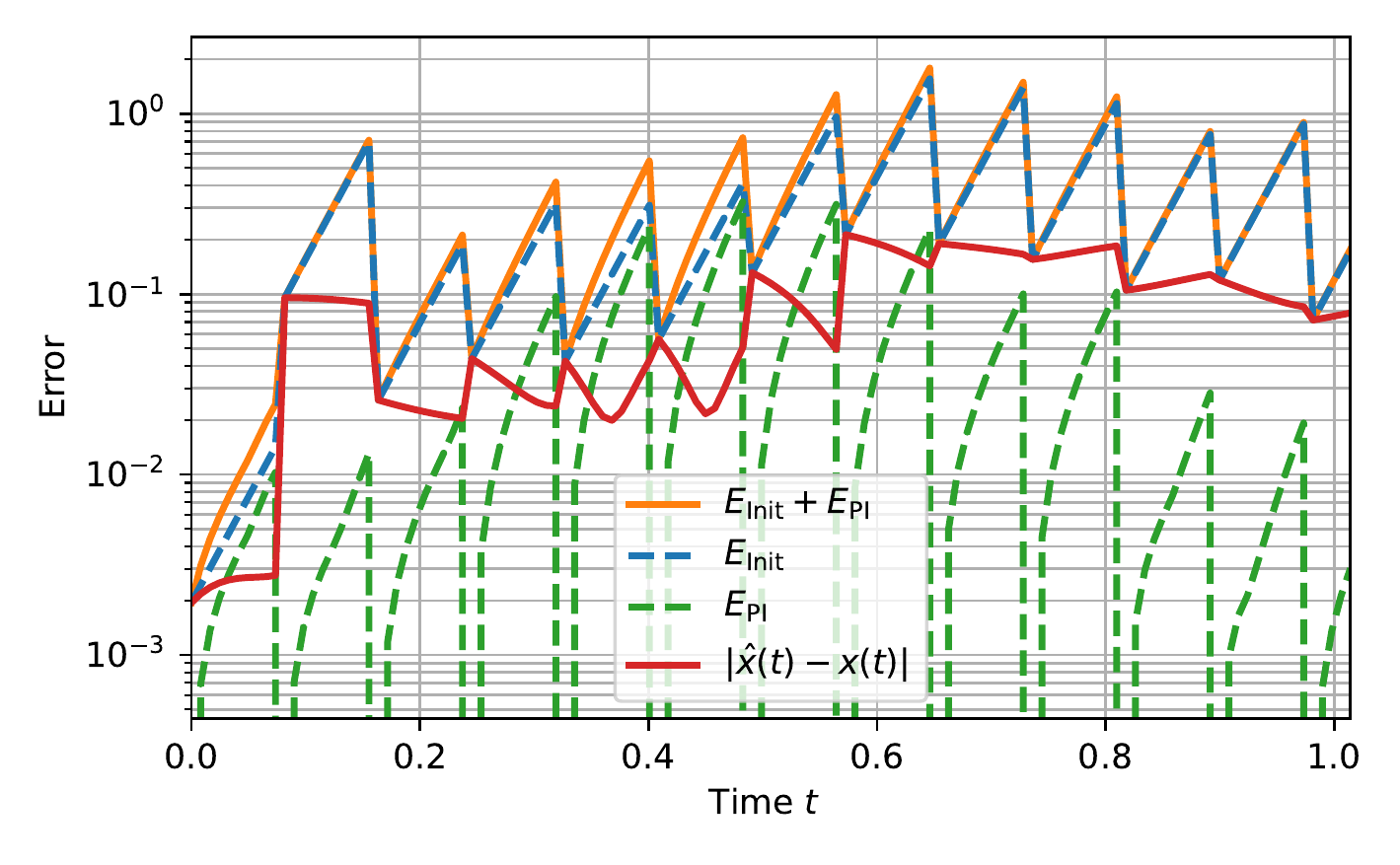}
	\caption{Error contributions for the first time unit of simulation and for comparison the absolute error (red line) between reference solution $x(t)$ and PINN based solution $\hat{x}(t)$.}
	\label{fig::first200_abserr}
\end{figure}

As in the 1D example we train the error NN to represent the error estimator. We used a set of 25.000 generated training points $Y_\mathrm{gen}$ in the allowed range (as described previously). The data-driven error NN was parametrized to have 8 layers with each 32 neurons utilizing tanh activation function and L-BFGS optimizer \cite{LiuN89}. The training was performed over 100.000 epochs. In Fig.~\ref{fig::error_nn_pendulum}, the error estimator according to Theorem~\ref{thm::core_result} is shown as a solid orange line, whereas the output of the error NN is shown as dashed blue line. In this case, there was no necessity to penalize underestimation additionally, rather slight overestimation of the error remained visible. This overestimation is considered acceptable, since it mostly preserves the magnitude. Especially for the strongly non-linear regime, in which the first time unit lies, a more dense data set would be required to train the error NN to retrieve more accurate approximations of the error. 

\begin{figure}
	\centering
	\includegraphics[width=0.97\linewidth]{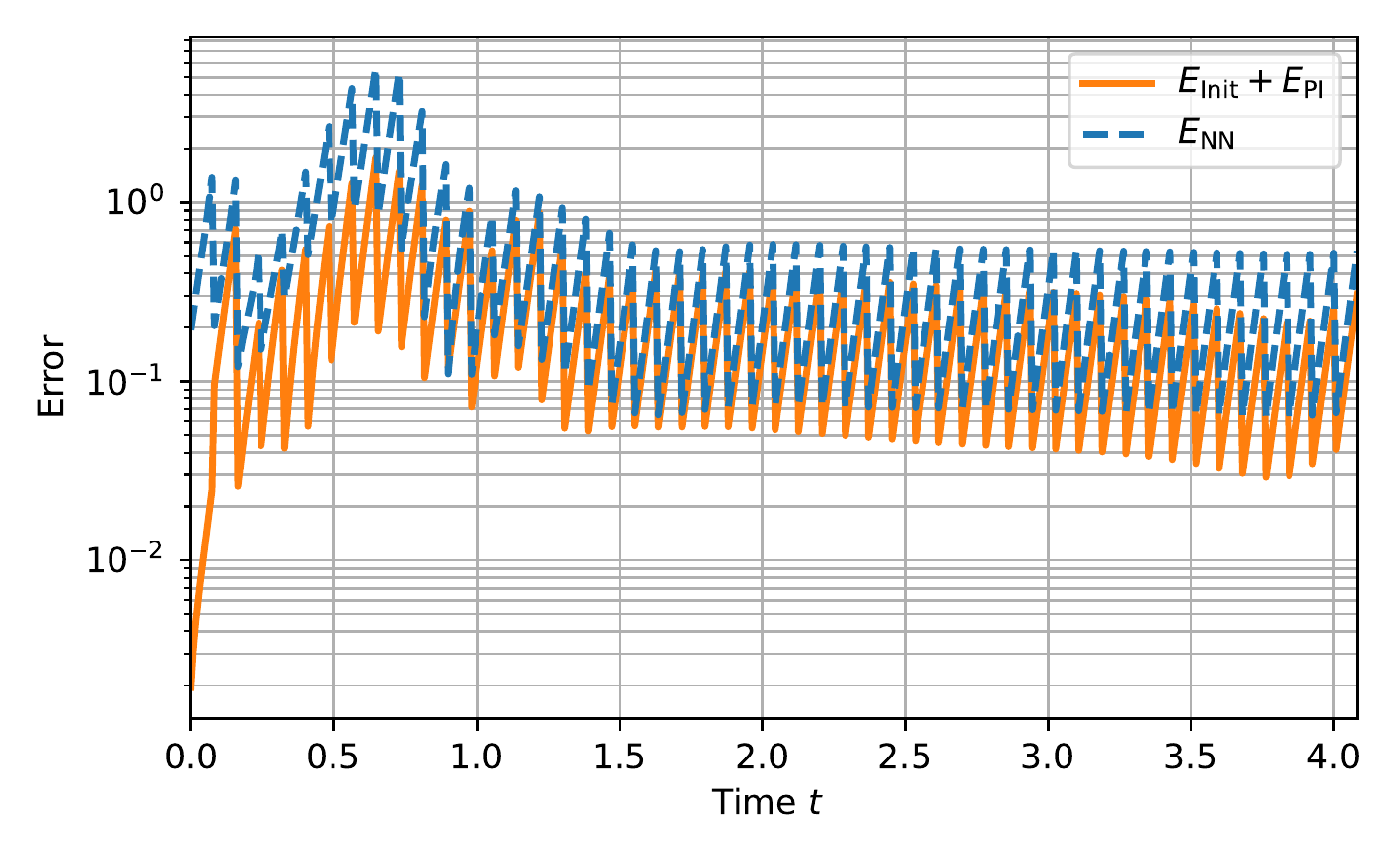}
	\caption{A posteriori error estimator for the test data set for the inverted pendulum (orange solid line) in comparison to the error indicator given by the error NN (blue dashed line). }
	\label{fig::error_nn_pendulum}
\end{figure}

Similar to the 1D example~\eqref{eq::exampleprob}, learning the error has shown to be challenging, which is reflected by the NN size required to retrieve good results. The scaling of the network was systematically increased from 2 layers with 4 neurons each to the size, of which the results were presented here. As Fig.~\ref{fig::error_nn_pendulum} suggests, there is a high regularity in the error estimator for, e.g., the dependence on time between control points, which is easily learned. However, the magnitude here poses a problem and is not appropriately reflected when using smaller networks.

\begin{remark}
	In the previous examples, results were presented which use the hyperbolic tangent as the activation function. However, the choice of the activation function does not affect the validity or the applicability of the estimator. Tests with other activation functions such as gaussian error linear unit (gelu), softmax, or sigmoid-weighted linear units (silu) showed similar results.
	The same holds for the choice of the optimizer. 
\end{remark}

\section{Discussion and Outlook}
\label{sec::discussion}
We have shown that for the approximation of ODEs, an a~posteriori certification of machine learning methods using reliable error estimators is possible . These error estimators are particularly useful since no knowledge about the true solution is needed. 

For linear systems, we derive a strong error bound, which has shown to be close in magnitude to the actual error by the 1D toy example presented in subsec.~\ref{subsec::Appl_1DODE}. For non-linear systems, the Lipschitz constant of the right-hand side of the ODE dominates the evolution of the predicted error over time. This has consequences on the practical use of this methodology. For systems with large Lipschitz constants, it might be necessary to introduce observation points to interrupt error aggregation and restart the error prediction. 

Many questions remain open after this initial investigation. Beginning with the basics of a posteriori error estimation \cite{Ver94}, a thorough investigation of the possibility to give lower bounds on the error estimator could show the applicability of the error estimators. Without an adequate lower bound, the risk remains that we significantly overestimate the error, as detailed in subsec~\ref{subsec::InvPendulum}.

As already observed in \cite{JagKK20} the choice of the activation function plays a more significant role for physics-informed machine learning than for conventional data-driven methods. It would be a valuable extension to systematically evaluate the impact of the choice of activation functions, hyperparameters and the introduction of optional hyperparameter (e.g., as done in \cite{JagKK20}) on the presented methodology. Moreover, the error estimators suggest introducing an additional weighting parameter into the physics-inspired part of the loss function, which subsequently allows for balancing the error contribution of the initial condition and the ODE residual. A systematic choice of $\eta$ is thus promising for the certification and may further improve the prediction ability of the neural network.

In addition to open questions about the applicability of these error bounds to ODE-defined problems, extensions to PDE-governed systems are essential (cf.~Remark~\ref{rem::extension_PDE}). Only this subsequent step makes a posteriori error estimators available to the majority of systems, which PINNs could address.

\bibliographystyle{IEEEtran}
\bibliography{literature}

\enddocument